\documentclass{article}

\usepackage[dblblindworkshop,final]{neurips_2026}
\workshoptitle{VLM4RWD}

\usepackage[utf8]{inputenc}
\usepackage[T1]{fontenc}
\usepackage{hyperref}
\usepackage{url}
\usepackage{booktabs}
\usepackage{amsfonts}
\usepackage{amsmath}
\usepackage{amssymb}
\usepackage{amsthm}
\usepackage{nicefrac}
\usepackage{microtype}
\usepackage{xcolor}
\usepackage{graphicx}
\usepackage{multirow}
\usepackage{array}
\usepackage{tikz}
\usepackage{pgfplots}
\usepackage{subcaption}
\usepackage{listings}
\usepackage{algorithm}
\usepackage{algorithmic}
\usepackage{enumitem}

\newtheorem{theorem}{Theorem}
\newtheorem{corollary}{Corollary}

\pgfplotsset{compat=1.18}
\usetikzlibrary{arrows.meta, positioning, shapes.geometric, fit, calc}
\definecolor{twocolor}{RGB}{180,50,50}
\definecolor{eightcolor}{RGB}{40,100,160}
\definecolor{neutralgray}{RGB}{90,90,90}

\title{Token-Set Choice Confounds POPE: A Systematic Audit of\\
Yes/No Extraction in VLM Hallucination Evaluation}

\author{%
  Kesav Kumar Jayakumar \\
  Department of Information Technology \\
  Sri Krishna College of Technology \\
  Coimbatore 641042, India \\
  \texttt{kesavk659@gmail.com} \\
  \And
  Karthigeyan Thilak \\
  Department of Computer Science and Engineering \\
  Sri Eshwar College of Engineering \\
  Coimbatore 641202, India \\
  \texttt{karthik.suthraye@gmail.com}
}

\begin{document}

\maketitle

\begin{abstract}
We set out to reproduce a recent hallucination-mitigation result on the
POPE benchmark. The LLaVA-1.5-7B baseline we measured, 82.21 F1 on the
adversarial split, sat 1.21 to 8.67 points above the adversarial baselines
we planned to compare against (81.00, 79.30, and 73.54). We looked at the
evaluation code and found a readout choice that moves F1 on its own. A
two-token logit readout decides yes versus no by
comparing two vocabulary IDs, 3582 (\texttt{yes}) and 1217
(\texttt{no}). Across 9{,}000 greedy-decode
questions spanning all three splits, neither ID is generated even once.
The model emits \texttt{`\,Yes'} (token 3869) and \texttt{`\,No'}
(token 1939), because SentencePiece prepends a space byte after the
prompt suffix \texttt{ASSISTANT:}. Read with an eight-token lookup that
covers the space-prefixed surface forms, LLaVA-1.5-7B scores F1 0.8221,
0.8498, and 0.8713 on the adversarial, popular, and random splits. That
6.13-point adversarial shift exceeds the adversarial gain reported by five
of the six papers we surveyed. None of them states its readout, and our
reference ran under 4-bit NF4 quantization, so we cannot
say how much of any reported gain a matched evaluation would keep. That
comparison stays unresolved until readout and precision are matched. On
the fixed baseline we run nine inference-time corrections of our own
(none beat it) and dissect why the adversarial split resists logit-level
intervention. In a 3{,}000-question probe, wrong predictions draw
\emph{more} image attention than correct ones (Cohen's $d=-0.46$,
$p=1.4\times10^{-24}$), the opposite of what grounding-based corrections
assume. A four-readout audit across four further VLMs from three
tokenizer families confirms the confound is protocol-level, not
model-specific: porting the LLaVA-1.5 IDs verbatim drops F1 to 0.00 on
the two disjoint-vocabulary models, whereas a tokenizer-derived readout
holds between 0.80 and 0.92. Code, 9{,}000 prediction records, and the
diagnostic script are available at
\url{https://github.com/Kesav2k04/ugaa-research}.
\end{abstract}

\section{Introduction}
\label{sec:intro}

Our starting point was ordinary. We picked POPE~\citep{li2023evaluating}
to benchmark a hallucination-mitigation idea of our own, and the first
step was to reproduce the LLaVA-1.5-7B~\citep{liu2023improved} baseline.
Running the model in 4-bit on a single NVIDIA RTX~3070~Ti gave us 82.21
F1 on POPE adversarial. The adversarial baselines we meant to beat sat at
73.54, 79.30, and 81.00. We looked at the evaluation code and found a
readout choice that moves F1 on its own: a two-token logit-readout
pattern. Whether any published method used it we could not confirm
(Section~\ref{sec:impact}).

The pattern grabs the first generated token, then labels it by comparing
two logits, $\ell_{3582}$ for \texttt{yes} against $\ell_{1217}$ for
\texttt{no}. We logged the argmax token on all 9{,}000 of our runs. Token
3582 appeared zero times. Token 1217 also appeared zero times. What the
model does produce, every single time, is token 3869
(\texttt{`\,Yes'}, space-prefixed) for a yes answer and token 1939
(\texttt{`\,No'}, space-prefixed) for a no. The reason is mechanical. The
LLaMA-2 SentencePiece tokenizer~\citep{kudo2018sentencepiece,touvron2023llama2}
glues a whitespace byte onto word-initial tokens in mid-sentence
position, and the \texttt{ASSISTANT:} suffix leaves the answer word
exactly there. IDs 3582 and 1217 are perfectly valid vocabulary entries.
They simply never surface after \texttt{ASSISTANT:} under greedy decoding
in this model.

Under greedy decoding the model never emits ID 3582 or 1217 at the
answer position, so the two-token rule compares two candidates the
decoder excludes. Take the 1{,}644 adversarial questions the model answers
\texttt{`\,No'} on: the 3582 logit beats the 1217 logit on 924 of them.
Of those, 660 sit on true-negative questions and flip into false
positives; the remaining 264 sit on borderline label-yes questions and
become true positives. Borderline positives flip at 0.810 and true
negatives at 0.501, so the excluded logits may still track the model's lean,
and we call the failure candidate exclusion. The yes-rate
climbs to 0.760, recall inflates to 0.9587, and F1 sinks to 0.7608.
Figure~\ref{fig:pipeline} follows the path end to end. Nothing in the
model's behaviour changed. The same generated tokens, read under the
eight-token rule, produce a different aggregate.

\begin{figure}[t]
\centering
\begin{tikzpicture}[
  box/.style={draw, rounded corners=3pt, fill=white,
              text width=2.6cm, align=center,
              minimum height=0.9cm, font=\small},
  arrow/.style={-{Stealth[length=5pt]}, thick, color=neutralgray},
  wrongbox/.style={draw, rounded corners=3pt,
                   fill=twocolor!10, draw=twocolor!60,
                   text width=3.4cm, align=center,
                   minimum height=1.0cm, font=\small,
                   inner sep=4pt},
  rightbox/.style={draw, rounded corners=3pt,
                   fill=eightcolor!10, draw=eightcolor!60,
                   text width=3.4cm, align=center,
                   minimum height=1.0cm, font=\small,
                   inner sep=4pt},
]
\node[box] (prompt) {\texttt{...ASSISTANT:}};
\node[box, right=1.2cm of prompt] (tok) {LLaMA tokenizer\\(SentencePiece)};
\node[box, right=1.2cm of tok] (gen)
  {generates \texttt{`\,Yes'}\\(ID~3869) or\\\texttt{`\,No'} (ID~1939)};

\draw[arrow] (prompt) -- (tok);
\draw[arrow] (tok) -- (gen);

\node[wrongbox, below=1.0cm of gen, xshift=-2.4cm] (twotok)
  {\textcolor{twocolor}{\textbf{2-token check}}\\
   IDs \{3582, 1217\}\\
   \textit{no match}};
\node[rightbox, below=1.0cm of gen, xshift=2.4cm] (eighttok)
  {\textcolor{eightcolor}{\textbf{8-token check}}\\
   IDs \{3582, 8241, 4874, \textbf{3869},\\
   1217, 3782, 694, \textbf{1939}\}\\
   \textit{matches correctly}};

\draw[arrow, color=twocolor]   (gen.south) -- ++(0,-0.35) -| (twotok.north);
\draw[arrow, color=eightcolor] (gen.south) -- ++(0,-0.35) -| (eighttok.north);

\node[wrongbox, below=0.7cm of twotok, minimum height=1.1cm] (twoout)
  {\textcolor{twocolor}{\textbf{wrong decision}}\\
   yes-rate $=$ 0.760\\F1 $=$ 0.761 (adv.)};
\node[rightbox, below=0.7cm of eighttok, minimum height=1.1cm] (eightout)
  {\textcolor{eightcolor}{\textbf{correct decision}}\\
   yes-rate $=$ 0.452\\F1 $=$ 0.822 (adv.)};

\draw[arrow, color=twocolor]   (twotok)   -- (twoout);
\draw[arrow, color=eightcolor] (eighttok) -- (eightout);

\node[font=\footnotesize\itshape, color=neutralgray]
  at ($(twoout.south)!0.1!(eightout.south) + (0,-0.55)$)
  {$\Delta_{\text{adversarial}} = 6.13$ F1 points};
\end{tikzpicture}
\caption{Decision path at the POPE answer position for LLaVA-1.5-7B.
The model always emits a space-prefixed surface form
(\texttt{`\,Yes'} or \texttt{`\,No'}). The two-token check inspects two
IDs the model never produces in this context, so every comparison is
between two candidates the decoder excluded. On the adversarial split
that comparison reads 924 of the model's 1{,}644 \texttt{`\,No'}
answers (56\%) as yes, pushing the yes-rate to 0.760 on a 50\%
positive split.}
\label{fig:pipeline}
\end{figure}
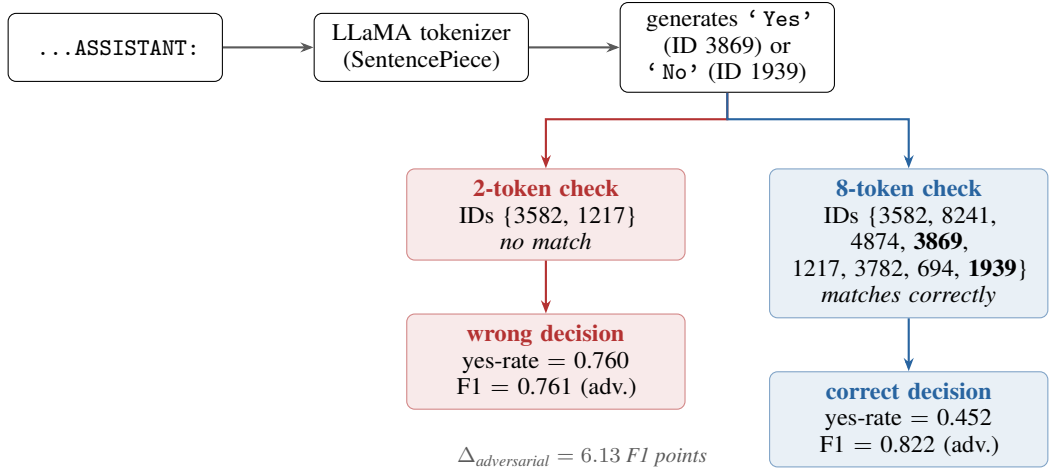

The 6.13-point shift matters because it exceeds the adversarial gain
reported by five of the six papers we surveyed, and none of those papers
states which readout produced its baseline. That leaves the published
comparison open. Setting our 82.21 beside another paper's baseline would
not give a corrected gain, because our reference ran under 4-bit NF4
quantization and a published baseline may differ from it in readout,
precision, prompt format, or decoding. We report what each paper states
and leave the comparison unresolved until those factors are matched.
We then turned the eight-token protocol on our own ideas and
ran nine inference-time corrections, six grounding-based and three
contrastive. Every one finished below 0.8221. The reasons are specific,
and Appendix~\ref{app:correction} walks through them with
per-question diagnostics. These corrections and the attention probe
(Appendix~\ref{app:why}) are secondary findings, and the readout audit
does not depend on them.

Three claims organize the paper. Section~\ref{sec:artifact} documents the
space-prefixing behaviour across all 9{,}000 questions and traces how the
protocol error propagates into reported metrics. Section~\ref{sec:impact}
tabulates the baselines and gains that published methods report and
explains why they cannot be set against ours until readout and precision
are matched.
Section~\ref{sec:crossmodel} extends the audit to four additional VLMs
and shows the confound is a property of the readout, not of any one
model. Extended derivations, the nine failed corrections, the formal
equivalence proof, and the full literature survey live in the appendix.

\section{Background}
\label{sec:background}

\subsection{The POPE benchmark}

\citet{li2023evaluating} proposed POPE as a polling alternative to
open-ended caption hallucination metrics, where drawing a clean binary
correctness line is awkward. Each poll shows one image and one target
object and asks: ``Is there a \{object\} in the image?'' Three
negative-sampling strategies give three splits. \emph{Random} draws
absent objects uniformly from COCO~\citep{lin2014microsoft}.
\emph{Popular} picks the most frequent COCO objects that happen to be
missing from the queried image. \emph{Adversarial} picks objects that
co-occur with objects that are present (a snowboard scene gets asked
about skis; a kitchen scene gets asked about a refrigerator). Every split
holds 3{,}000 questions, balanced 50/50 between yes and no.

The distinction that matters for our claim is how the yes/no answer gets
read off, and POPE-style evaluation admits several readouts. We separate
four. (i) The \emph{official POPE benchmark}~\citep{li2023evaluating},
which we treat as the dataset plus the metric definition (yes/no F1,
precision, recall, and yes-rate over the balanced splits). (ii)
\emph{String-parsing evaluation}, where a generation step produces text
and the label is recovered by searching for ``yes'' or ``no''; this is
blind to which LLaMA token ID the model chose. (iii) A \emph{two-token
logit readout}, comparing the first-token logits at IDs 3582 and 1217.
We do not claim the official POPE code mandates this readout, and we could
not verify the readout of any published baseline we survey
(Section~\ref{sec:impact}). (iv) Our \emph{eight-token logit readout},
which adds the alternate-case and space-prefixed forms so the comparison
reads tokens the tokenizer actually emits after \texttt{ASSISTANT:}.

\subsection{LLaVA-1.5 and its tokenizer}

LLaVA-1.5~\citep{liu2023improved} wires a CLIP-ViT-L/14-336 vision
encoder~\citep{radford2021learning} to Vicuna-1.5-7B through a two-layer
MLP. Vicuna-1.5 descends from LLaMA-2~\citep{touvron2023llama2}, whose
SentencePiece subword tokenizer~\citep{kudo2018sentencepiece} handles
word boundaries by prepending a whitespace byte to word-initial tokens in
mid-sentence position. So ``Yes'' owns two IDs in the LLaMA vocabulary:
8241 for the start-of-sentence variant (\texttt{Yes}) and 3869 for the
mid-sentence variant (\texttt{`\,Yes'}). The same split holds for ``yes''
(3582 vs.\ 4874), ``No'' (3782 vs.\ 1939), and ``no'' (1217 vs.\ 694).
The LLaVA-1.5 template ends with the literal string \texttt{ASSISTANT:},
so the next token is mid-sentence, the space byte is prepended, and the
model emits 3869 or 1939. Appendix~\ref{app:tokens} lists the raw decode
outputs.

\subsection{Evaluation-protocol sensitivity}

Scores that drift with measurement choices are a familiar problem in NLP
evaluation. \citet{liang2022holistic} showed benchmark numbers move
visibly with prompt format and few-shot setup;
\citet{zheng2023judging} traced answer-ordering bias in LLM-as-judge
protocols. \citet{li2026jev} found that a judge's label probabilities
shift with the output type a question is asked through, and compared
answer extraction with direct adjudication.
Our finding belongs to that family: a small, unexamined detail
of how evaluation reads model output quietly becomes load-bearing for
downstream comparisons. We keep the scope explicit. The conditions are a
LLaMA-family tokenizer, the \texttt{ASSISTANT:} template, and greedy
decoding. Methods that run generation and parse the output string are not
touched by the two-token readout. Methods that inherit the two-ID lookup
are, by an amount that tracks the share of the negative split the model
answers ``no'' on (Eq.~\ref{eq:gap-closed}).

\section{The Token-Extraction Artifact}
\label{sec:artifact}

\subsection{Experimental setup}
\label{sec:setup}

We ran LLaVA-1.5-7B in 4-bit NF4 quantization via
BitsAndBytes~\citep{dettmers2023qlora} on an NVIDIA RTX~3070~Ti with
8~GB VRAM, on \texttt{transformers v4.40.1} with
\texttt{torch.manual\_seed(42)}. Each question called
\texttt{model.generate} with \texttt{max\_new\_tokens=1} and
\texttt{output\_scores=True}. For each question we saved the argmax token
ID and the decision of each readout below. The logits were compared in
memory and not written to disk, so the records hold the sign of
$\ell_{3582} - \ell_{1217}$ but neither value. A separate pass over the
adversarial split saved only the maximum logit within each eight-token
set (Appendix~\ref{app:math}). Greedy decoding has no sampling step, and
that second, instrumented pass reproduced the sweep's eight-token
decision on 2{,}998 of the 3{,}000 adversarial questions; the other two
sit within one fp16 step of a tie (Appendix~\ref{app:why}).
The full 9{,}000-question sweep took 76.3 minutes of wall-clock time,
0.508 seconds per question on this hardware.

\paragraph{Two-token protocol.} The original rule:
$\hat{y} = \mathbf{1}[\ell_{3582} > \ell_{1217}]$.

\paragraph{Eight-token protocol.} The rule we recommend:
$\hat{y} = \mathbf{1}[\max_{i \in Y} \ell_i > \max_{j \in N} \ell_j]$,
with $Y = \{3582, 8241, 4874, 3869\}$ (surface forms \texttt{yes},
\texttt{Yes}, \texttt{`\,yes'}, \texttt{`\,Yes'}) and
$N = \{1217, 3782, 694, 1939\}$ (\texttt{no}, \texttt{No},
\texttt{`\,no'}, \texttt{`\,No'}). The max aggregator matches greedy
behaviour: a greedy decoder emits the single highest-logit token, and
that token lands in $Y$ exactly when $\max_Y \ell > \max_N \ell$. A sum
or softmax aggregator would weight all eight forms, including ones the
model never samples, and drift away from the decoder's own rule.

\subsection{What the model actually generates}

Table~\ref{tab:tokens} gives the first-token distribution across the
three splits. Two IDs carry every question: 3869 (\texttt{`\,Yes'}) and
1939 (\texttt{`\,No'}). The other six candidate tokens appear zero times.
Tokens 3582 and 1217 are never the argmax on any of the 9{,}000
questions. That is all the zero counts establish: both IDs lose to a
space-prefixed form every time, and the records hold neither of their
logits (Section~\ref{sec:setup}). They do hold the sign of
$\ell_{3582} - \ell_{1217}$, and that sign tracks the emitted answer. On
all 3{,}893 questions where the model emits \texttt{`\,Yes'}, the
two-token rule also says yes. On \texttt{`\,No'} emissions it says yes on
0.810 of the adversarial label-yes questions and 0.501 of the label-no
ones (Section~\ref{sec:math}). The excluded logits may track the model's
lean, so we treat the failure as candidate exclusion and make no claim
about how large those logits are.

\begin{table}[t]
\caption{First generated token counts across all POPE splits.
LLaVA-1.5-7B produces only space-prefixed forms. The two-token protocol
reads the two rightmost columns, zero in every split.}
\label{tab:tokens}
\centering
\begin{tabular}{lrrrr}
\toprule
Split & \texttt{`\,Yes'} (3869) & \texttt{`\,No'} (1939)
      & \texttt{yes} (3582) & \texttt{no} (1217) \\
\midrule
Adversarial & 1{,}356 & 1{,}644 & 0 & 0 \\
Popular     & 1{,}303 & 1{,}697 & 0 & 0 \\
Random      & 1{,}234 & 1{,}766 & 0 & 0 \\
\midrule
Total       & 3{,}893 & 5{,}107 & 0 & 0 \\
\bottomrule
\end{tabular}
\end{table}

\subsection{F1 under each protocol}

Table~\ref{tab:protocol} pairs the two protocols on the same set of
predictions. The eight-token rule recovers the token greedy decoding
would emit; the two-token rule compares two entries the decoder never
selects. On the adversarial split the two-token readout
reaches near-perfect recall (0.9587) by riding a 0.760 yes-bias, and
precision pays for it, falling to 0.6307.
Appendix~\ref{app:mass} shows the full confusion matrices and the running
yes-rate curves behind these numbers.

\begin{table}[t]
\caption{POPE results for LLaVA-1.5-7B under both protocols. Two-token
evaluation buys recall on the adversarial split through a 0.760 yes-bias,
not accurate detection; the precision row shows how much of that coverage
is wrong.}
\label{tab:protocol}
\centering
\begin{tabular}{llccccc}
\toprule
Split & Protocol & F1 & Precision & Recall & Yes-rate & $\Delta$ \\
\midrule
\multirow{2}{*}{Adversarial}
 & Two-token   & 0.7608 & 0.6307 & 0.9587 & 0.760 &
   \multirow{2}{*}{+0.0613} \\
 & Eight-token & 0.8221 & 0.8658 & 0.7827 & 0.452 & \\
\addlinespace
\multirow{2}{*}{Popular}
 & Two-token   & 0.7961 & 0.6797 & 0.9607 & 0.707 &
   \multirow{2}{*}{+0.0537} \\
 & Eight-token & 0.8498 & 0.9140 & 0.7940 & 0.434 & \\
\addlinespace
\multirow{2}{*}{Random}
 & Two-token   & 0.8397 & 0.7459 & 0.9607 & 0.644 &
   \multirow{2}{*}{+0.0316} \\
 & Eight-token & 0.8713 & 0.9652 & 0.7940 & 0.411 & \\
\bottomrule
\end{tabular}
\end{table}

\subsection{The arithmetic in one line}
\label{sec:math}

The 6.13-point shift is not mysterious; it falls straight out of
confusion-matrix arithmetic once the protocol error is written down. The
eight-token rule preserves the generated decision, giving
$\mathrm{TP}_8 = T$, $\mathrm{FP}_8 = N - G$, $\mathrm{FN}_8 = P - T$,
and $\mathrm{TN}_8 = G$, where $T$ is the true-yes count and $G$ the
true-negative count. The two-token rule flips a fraction $\rho_N$ of the
true-negative \texttt{`\,No'} emissions into false positives and lifts a
fraction $\varphi$ of the false-negative \texttt{`\,No'} emissions into
true positives. It never turns an emitted \texttt{`\,Yes'} into a no:
the paired records show zero such transitions among the 3{,}893
\texttt{`\,Yes'} emissions on the three splits
(Table~\ref{tab:transitions}). Under that observed condition the gap has
a closed form:
\begin{equation}
\begin{aligned}
\Delta_{\mathrm{F1}} = \mathrm{F1}_8 - \mathrm{F1}_2
   = {} & \frac{2T}{2T + (N - G) + (P - T)} \\
   & - \frac{2\big(T + \varphi(P{-}T)\big)}
        {2\big(T + \varphi(P{-}T)\big) + (N - G) + \rho_N G + (1{-}\varphi)(P - T)}.
\end{aligned}
\label{eq:gap-closed}
\end{equation}
Both rates are empirical transition rates, counted separately on each
split, and every split has $P{=}N{=}1{,}500$. On the adversarial split
the direct counts fix every term: $T{=}1{,}174$, $G{=}1{,}318$,
$\rho_N = 660/1{,}318 = 0.501$ and $\varphi = 264/326 = 0.810$. This
reproduces $\mathrm{F1}_8 = 0.8221$ and $\mathrm{F1}_2 = 0.7608$ exactly.
Popular has $T{=}1{,}191$, $G{=}1{,}388$, $\rho_N = 567/1{,}388 = 0.409$
and $\varphi = 250/309 = 0.809$, so
$\mathrm{TP}_2 = 1{,}191 + 250 = 1{,}441$,
$\mathrm{FP}_2 = 112 + 567 = 679$, $\mathrm{FN}_2 = 59$ and
$\mathrm{F1}_2 = 2{,}882/3{,}620 = 0.7961$. Random has the same $T$ and
$\varphi$ but $G{=}1{,}457$ and $\rho_N = 448/1{,}457 = 0.307$, so
$\mathrm{FP}_2 = 43 + 448 = 491$ and
$\mathrm{F1}_2 = 2{,}882/3{,}432 = 0.8397$. The rate $\varphi$ stays near
0.81 on every split, while $\rho_N$ falls from 0.501 to 0.307 as the
eight-token false positives drop from 182 to 43. Plugging the
adversarial rates into the popular and random counts gives 0.7690 and
0.7761, well off the observed values, so the rates do not transfer
between splits. The reconstruction is an accounting identity built from
transitions counted on the same records. It confirms the bookkeeping; it
does not test the mechanism independently, which would need the
individual logits the records lack. Appendix~\ref{app:math} gives the
full derivation, the paired transition table and the notation.

\section{Comparability with Published Results}
\label{sec:impact}

Table~\ref{tab:reeval} lists the baseline, result, and gain that each
surveyed paper reports for LLaVA-1.5-7B on POPE. None of these papers
states which readout it used, and we do not hold their per-question logs,
so the table cannot tell us whether a given baseline came from the
two-token logits, the eight-token set, or parsed generated text. Our 82.21
cannot stand in for those baselines either. It ran under 4-bit NF4
quantization, and the published numbers may come from higher-precision
runs, so a gap between our reference and a published baseline would mix
the readout effect with a precision difference. The table therefore
computes no gain against 82.21 and assigns no readout to any paper.

\begin{table}[t]
\caption{Baselines, results, and gains as each paper reports them for
LLaVA-1.5-7B on POPE (F1$\times100$). The ``Readout'' column records what
each paper states. Our 82.21 reference ran under 4-bit NF4 quantization,
so a comparison with published baselines measured at higher precision
would mix readout effects with precision differences.}
\label{tab:reeval}
\centering
\small
\setlength{\tabcolsep}{4pt}
\begin{tabular}{@{}lcccl@{}}
\toprule
Method & Rep.~base. & Rep.~res. & Rep.~gain & Readout \\
\midrule
Tang~et~al.~\citep{visflow2025} (adv.)  & 73.54 & 84.35 & $+10.81$ & not stated \\
Tang~et~al.~\citep{visflow2025} (pop.)  & 76.53 & 87.09 & $+10.56$ & not stated \\
Tang~et~al.~\citep{visflow2025} (rand.) & 81.54 & 89.55 & $+8.01$  & not stated \\
Seo~et~al.~\citep{seo2025} (adv.)       & 79.30 & 80.00 & $+0.70$  & not stated \\
Chen~et~al.~\citep{adaptvis2025} (adv.) & 81.00 & 81.80 & $+0.80$  & not stated \\
Huo~et~al.~\citep{huo2024sid} (adv.)    & 80.90 & 82.50 & $+1.60$  & not stated \\
Jiang~et~al.~\citep{jiang2026pnd} (adv.)& 75.06 & 80.82 & $+5.76$  & not stated \\
Jana~et~al.~\citep{jana2026bracs} (adv.)& 82.07 & 83.03 & $+0.96$  & not stated \\
\bottomrule
\end{tabular}
\end{table}

The adversarial gains in the table range from $+0.70$ to $+10.81$. Five
of the six papers report less than the 6.13 points by which the readout
alone moved our adversarial F1; Tang et al.\ report $+10.81$. The
adversarial baselines run from 73.54 (Tang et al.) to 82.07 (Jana et al.),
a spread of 8.53 points on one frozen checkpoint, and the baselines of
Tang, Seo, and Chen (73.54, 79.30, 81.00) already span 7.46. Readout,
prompt format, quantization, and decoding are possible sources of such a
spread. The published numbers do not let us separate them, and
separating them takes matched reruns of each method with the readout and
the precision fixed and logged.

\citet{visflow2025} report gains of $+10.81$, $+10.56$, and $+8.01$ on the
adversarial, popular, and random splits, from baselines of 73.54, 76.53,
and 81.54. We found no public implementation for that paper
(arXiv:2506.12609) as of June 2026. \citet{seo2025} and
\citet{adaptvis2025} report adversarial gains of $+0.70$ and $+0.80$ from
baselines of 79.30 and 81.00. Each gain compares a method with a baseline
from the same paper, so it is informative whenever both numbers share one
readout and one precision. The papers do not say whether they do. If a
method parsed generated strings, the readout we audit may not affect its
numbers at all. Nothing here shows that any of these methods fails to help
LLaVA-1.5-7B, or that it helps by less than reported. The size of each
gain under a matched protocol stays unknown. A one-line statement of the
readout in each paper, or a released evaluation script, would settle which
readout produced each baseline.

\section{Multi-Model Validation}
\label{sec:crossmodel}

The audit so far concerns one model. To separate ``LLaVA-1.5 quirk'' from
``readout property,'' we ran the identical four-readout audit on four more
VLMs spanning three tokenizer families: LLaVA-1.6-Mistral-7B (Mistral
vocabulary), InstructBLIP-Vicuna-7B and mPLUG-Owl2-LLaMA2-7B (LLaMA
family), and Qwen2-VL-7B-Instruct (Qwen vocabulary). Each model saw all
three splits at 3{,}000 questions per split, 9{,}000 per model.
LLaVA-1.6-Mistral ran locally (RTX~3070~Ti Laptop, 4-bit NF4,
\texttt{transformers v4.44.2}); the other three ran on Kaggle Tesla~T4
GPUs under the same audit script. Qwen2-VL used its native chat template,
since it was never trained on the LLaVA-1.5 prompt string; the other three
used the LLaVA-1.5 template. Table~\ref{tab:multimodel} is the headline
result, and two patterns hold across every model and split.

\begin{table}[t]
\caption{POPE F1 under the four readouts for four VLMs beyond the
LLaVA-1.5-7B reference (top block, from Table~\ref{tab:protocol}).
\emph{legacy\_2tok} and \emph{legacy\_8tok} hard-code the LLaVA-1.5
LLaMA-2 IDs; \emph{dynamic\_single} re-derives the single-token yes/no
IDs from each model's own tokenizer; \emph{string\_parse} greedily
decodes and parses the first \texttt{yes}/\texttt{no} substring. Bold
marks the F1 a careful practitioner would report.}
\label{tab:multimodel}
\centering
\small
\setlength{\tabcolsep}{4pt}
\resizebox{\linewidth}{!}{%
\begin{tabular}{llcccc}
\toprule
Model (tokenizer) & Split & legacy\_2tok & legacy\_8tok & dynamic\_single & string\_parse \\
\midrule
\multirow{3}{*}{LLaVA-1.5-7B (LLaMA-2)}
 & adversarial & 0.7608 & \textbf{0.8221} & 0.8221 & 0.8221 \\
 & popular     & 0.7961 & \textbf{0.8498} & 0.8498 & 0.8498 \\
 & random      & 0.8397 & \textbf{0.8713} & 0.8713 & 0.8713 \\
\midrule
\multirow{3}{*}{LLaVA-1.6-Mistral-7B (Mistral)}
 & adversarial & 0.6667 & 0.0000 & \textbf{0.8521} & 0.8521 \\
 & popular     & 0.6667 & 0.0000 & \textbf{0.8953} & 0.8953 \\
 & random      & 0.6667 & 0.0000 & \textbf{0.9163} & 0.9163 \\
\midrule
\multirow{3}{*}{InstructBLIP-Vicuna-7B (LLaMA)}
 & adversarial & 0.7074 & 0.8183 & \textbf{0.8183} & 0.6667 \\
 & popular     & 0.7277 & 0.8553 & \textbf{0.8553} & 0.6667 \\
 & random      & 0.7613 & 0.8815 & \textbf{0.8815} & 0.6667 \\
\midrule
\multirow{3}{*}{mPLUG-Owl2-LLaMA2-7B (LLaMA-2)}
 & adversarial & 0.7754 & 0.8001 & \textbf{0.8001} & 0.8001 \\
 & popular     & 0.8060 & 0.8312 & \textbf{0.8312} & 0.8312 \\
 & random      & 0.8466 & 0.8707 & \textbf{0.8707} & 0.8707 \\
\midrule
\multirow{3}{*}{Qwen2-VL-7B-Instruct (Qwen)}
 & adversarial & 0.6674 & 0.0000 & \textbf{0.8457} & 0.8457 \\
 & popular     & 0.6671 & 0.0000 & \textbf{0.8578} & 0.8578 \\
 & random      & 0.6667 & 0.0000 & \textbf{0.8671} & 0.8671 \\
\bottomrule
\end{tabular}%
}
\end{table}

\begin{enumerate}[leftmargin=*,nosep]
\item \textbf{Porting the LLaVA-1.5 IDs verbatim is unsafe, and how badly
it fails depends on the target vocabulary.} The \emph{legacy\_8tok}
readout scores F1 0.00 on LLaVA-1.6-Mistral and Qwen2-VL: their
vocabularies put no yes/no mass on those IDs, so a comparison between two
never-generated logits resolves the same way every time. On mPLUG-Owl2 and
InstructBLIP, both LLaMA-family models that share the IDs, the fixed set
still works. The failure is vocabulary mismatch, not something universal.
\item \textbf{A readout derived from each model's own tokenizer is
robust.} The \emph{dynamic\_single} readout, comparing max-logits over the
single-token yes/no IDs the tokenizer actually produces, stays between
0.80 and 0.92 on every model and split. The spread across models is
genuine accuracy variation, not an artifact of the readout.
\end{enumerate}

The \emph{string\_parse} column matches \emph{dynamic\_single} to four
decimals on LLaVA-1.6-Mistral, mPLUG-Owl2, and Qwen2-VL. On InstructBLIP
it collapses to the 0.6667 all-yes floor across all three splits while
\emph{dynamic\_single} holds at 0.82 to 0.88. The cause is the parse, not
the model: the text our harness decoded for InstructBLIP did not begin
with the answer token, so the first \texttt{yes}/\texttt{no} substring
came from the wrong place. Appendix~\ref{app:crossmodel} contains the
token-ID map, the prompt-template robustness check, provenance notes, and
the secondary verification run.

\section{Practitioner Recommendations}
\label{sec:recs}

The audit reduces to a handful of checks that any POPE-style evaluation on
LLaMA-family VLMs can adopt without touching the dataset, the model, or
the metric.

\begin{enumerate}[leftmargin=*,nosep]
\item \textbf{Log the generated token ID and its decoded string for every
question.} A first-token argmax and its surface form are cheap to record
and remove any need to guess the readout convention afterwards.
\item \textbf{Verify the yes/no token sets under the exact prompt
template.} Different suffixes push first-token mass onto different forms.
The \texttt{ASSISTANT:} suffix we used puts it on the space-prefixed
\texttt{`\,Yes'}/\texttt{`\,No'} pair.
\item \textbf{Prefer the tokenizer-derived single-token logit readout.}
It depends only on the logit vector, never on what text a generation loop
returns. String parsing is equivalent to it (Theorem~\ref{thm:equiv},
Appendix~\ref{app:equiv}) only when it reads exactly the first generated
token, which is where InstructBLIP broke.
\item \textbf{Release prediction-level records} so reviewers can recompute
any aggregate and check the readout convention independently.
\end{enumerate}

Listing~\ref{lst:snippet} is the readout we recommend. It derives the
yes/no IDs from the model's own tokenizer, generates one token, and
compares the max-logits, so the same code runs unchanged across all five
models here. It adds no generation cost beyond encoding four short strings
once at setup.

\begin{lstlisting}[caption={Drop-in POPE answer readout. Derive the yes/no
token IDs from the model's own tokenizer, then compare first-token
max-logits. Runs unchanged across tokenizer families.},
label={lst:snippet}]
tok = processor.tokenizer
enc = lambda s: tok.encode(s, add_special_tokens=False)
yes_forms = ["yes", "Yes", " yes", " Yes"]
no_forms  = ["no",  "No",  " no",  " No"]
yes_ids = [enc(s)[0] for s in yes_forms if len(enc(s)) == 1]
no_ids  = [enc(s)[0] for s in no_forms  if len(enc(s)) == 1]
out = model.generate(**inputs, max_new_tokens=1,
                     do_sample=False,
                     return_dict_in_generate=True,
                     output_scores=True)
logit = out.scores[0][0]
y = max(logit[i] for i in yes_ids)
n = max(logit[j] for j in no_ids)
pred = "yes" if y > n else "no"
\end{lstlisting}

\section{Limitations and Conclusion}
\label{sec:conclusion}

The mismatch we report has a narrow technical surface: one logit
comparison between two IDs that LLaVA-1.5-7B does not produce after
\texttt{ASSISTANT:} under greedy decoding. The reach of that mistake
depends on which methods adopt the two-token readout. On POPE adversarial
F1 the shift to the eight-token readout is 6.13 points, larger than the
adversarial gain reported by five of the six papers we surveyed. None of
those papers states its readout, so we cannot say which of their baselines
the artifact touches, and we estimate no corrected gains. Our 82.21
reference also ran under 4-bit NF4 quantization, so setting it beside
baselines measured at higher precision would mix readout and precision
effects. No method is shown to be broken here. Numbers from different
studies cannot be compared at face value until the readout and the
precision are matched and reported. The correction is a one-line change in
the evaluation code.

A few caveats bound the result. The main 9{,}000-question audit uses one
checkpoint at one quantization level (4-bit NF4), so we cannot say whether
the 6.13-point shift is the same at higher precision. The nine corrections
(Appendix~\ref{app:correction}) were run only on the adversarial split.
The whitespace-prefix behaviour is shared across LLaMA-family tokenizers,
but the specific IDs differ by vocabulary, so it is the diagnostic in
Algorithm~\ref{alg:eight-token}, not the fixed eight-token set, that
carries the finding to a new model. Our nine inference-time corrections
did not beat the corrected baseline, and the per-signal analysis suggests
the visual-grounding and attention signals those methods rely on do not
discriminate the cases the adversarial split is built to query. Beating
that baseline may need training-time changes, not decoding-time ones. We
encourage future POPE work to verify the first-token behaviour of its
specific model-template pairing before fixing a readout token set.

\begin{ack}
We used Claude (Anthropic) as a coding assistant for experimental script
development and debugging. All experimental design, analysis,
interpretation, and writing are our own. This work received no external
funding. Hardware: NVIDIA RTX~3070~Ti Laptop (8~GB VRAM) for the primary
LLaVA-1.5 and LLaVA-1.6-Mistral runs; NVIDIA RTX~3050 Laptop (4~GB VRAM)
for verification; Kaggle Tesla~T4~$\times$~2 for the mPLUG-Owl2,
InstructBLIP-Vicuna, and Qwen2-VL evaluations.
\end{ack}

\bibliographystyle{plainnat}
\bibliography{references}

\newpage
\appendix


\section{Token ID Reference and Tokenizer Traces}
\label{app:tokens}

Table~\ref{tab:tokenids} lists all eight yes/no token IDs. The decode
outputs below were run in Python against the LLaMA tokenizer loaded
straight from the LLaVA-1.5-7B repository, and they confirm what each ID
produces:

\begin{lstlisting}
from transformers import AutoTokenizer
name = "llava-hf/llava-1.5-7b-hf"
tok = AutoTokenizer.from_pretrained(name)

tok.decode([3582])  # -> 'yes'
tok.decode([8241])  # -> 'Yes'
tok.decode([4874])  # -> ' yes'  (space-prefixed)
tok.decode([3869])  # -> ' Yes'  (space-pref) *generated*
tok.decode([1217])  # -> 'no'
tok.decode([3782])  # -> 'No'
tok.decode([694])   # -> ' no'   (space-prefixed)
tok.decode([1939])  # -> ' No'   (space-pref) *generated*
\end{lstlisting}

The two IDs marked \texttt{*generated*} are the only ones that turned
up as first tokens across all 9{,}000 questions.

\begin{table}[h]
\caption{Yes/no token IDs for LLaMA-family models. $\star$: the only IDs
LLaVA-1.5-7B generated after the \texttt{ASSISTANT:} prompt in our runs.}
\label{tab:tokenids}
\centering
\begin{tabular}{lll}
\toprule
Answer & Token ID & Decoded surface form \\
\midrule
yes & 3582 & \texttt{yes} \\
yes & 8241 & \texttt{Yes} \\
yes & 4874 & \texttt{`\,yes'} \\
yes & 3869\,$\star$ & \texttt{`\,Yes'} \\
\midrule
no  & 1217 & \texttt{no} \\
no  & 3782 & \texttt{No} \\
no  &  694 & \texttt{`\,no'} \\
no  & 1939\,$\star$ & \texttt{`\,No'} \\
\bottomrule
\end{tabular}
\end{table}

\section{Where the Mass Goes: Confusion Matrices and Convergence}
\label{app:mass}

Figure~\ref{fig:confusion} contrasts the full $2\times2$ confusion
matrices for adversarial POPE under both protocols. The eight-token matrix
is close to balanced (TP 1{,}174, TN 1{,}318, FP 182, FN 326). Switch to
two tokens and the false-positive count jumps from 182 to 842, as 660
true-negative \texttt{`\,No'} emissions flip to yes; the other 264 of the
924 flipped emissions land on borderline true-yes questions and become
true positives, which is why precision drops while recall climbs.

\begin{figure}[h]
\centering
\includegraphics[width=0.86\linewidth]{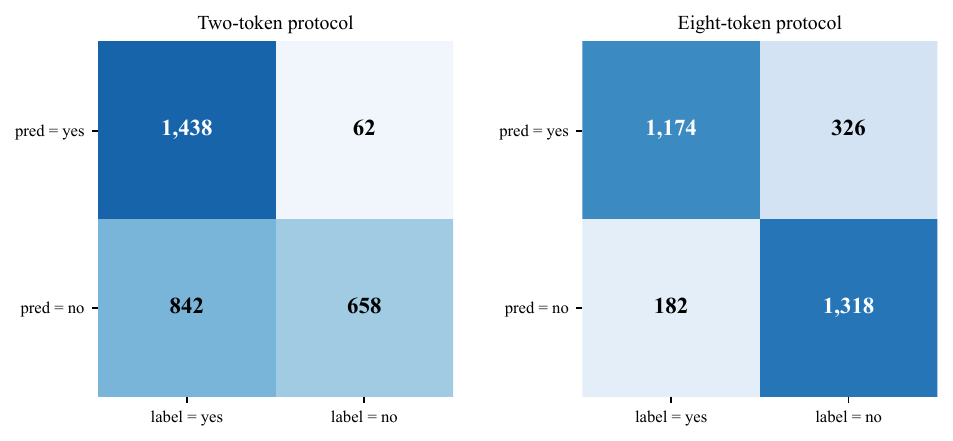}
\caption{Confusion matrices for POPE adversarial (3{,}000 questions) under
each protocol. Under two-token evaluation, 660 of the 1{,}318 true-No
predictions migrate to the FP cell, leaving precision at 0.6307. Under
eight-token evaluation only 182 stay there, and precision is 0.8658. The
model's behaviour is identical across both matrices; only the readout
changes.}
\label{fig:confusion}
\end{figure}

Figure~\ref{fig:convergence} tracks the running yes-rate of each protocol
over the 3{,}000-question adversarial stream. The eight-token line settles
at 0.452 by question 600 and holds, sitting close to the ground-truth
0.500; the two-token line settles at 0.760.
Figure~\ref{fig:allsplits-yesrate} repeats the comparison across all three
splits, and Figure~\ref{fig:toptokens} plots the first-token frequency for
all eight surface forms.

\begin{figure}[h]
\centering
\includegraphics[width=0.80\linewidth]{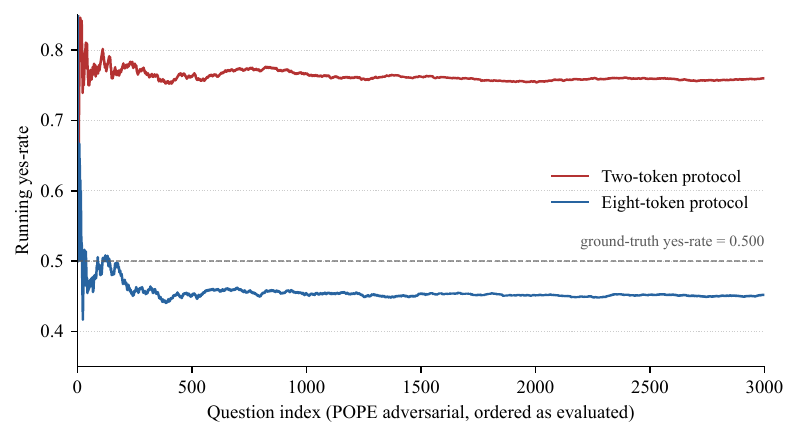}
\caption{Running yes-rate over the 3{,}000-question POPE adversarial
stream. The eight-token protocol (blue) stabilises near the ground-truth
yes-rate (dashed). The two-token protocol (red) stabilises around 0.760
because it reads 924 of the model's 1{,}644 \texttt{`\,No'} emissions
(56\%) as yes.}
\label{fig:convergence}
\end{figure}

\begin{figure}[h]
\centering
\includegraphics[width=0.92\linewidth]{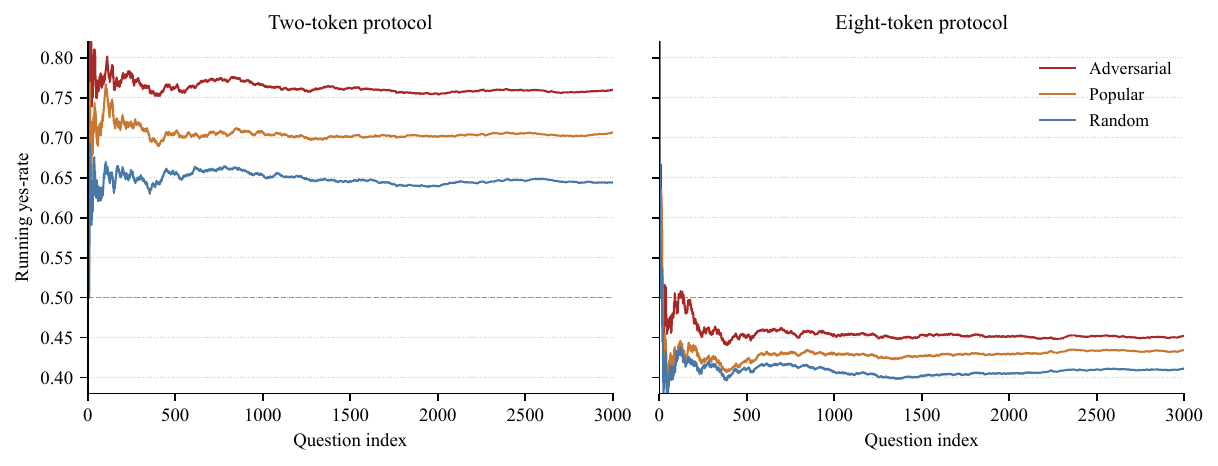}
\caption{Running yes-rate across all three POPE splits under both
protocols. The two-token protocol (left) holds a yes-bias on every split,
settling at 0.760, 0.707 and 0.644 on adversarial, popular and random;
the eight-token protocol (right) tracks the
underlying ground-truth mix.}
\label{fig:allsplits-yesrate}
\end{figure}

\begin{figure}[h]
\centering
\includegraphics[width=0.82\linewidth]{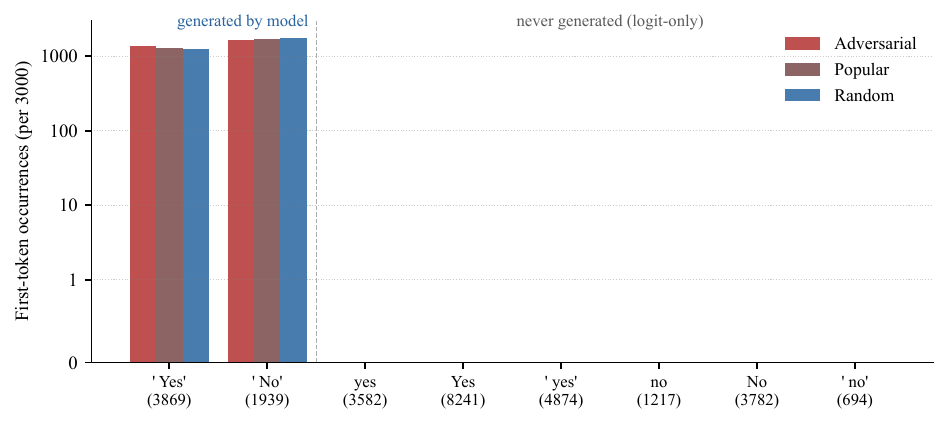}
\caption{First-token frequency at the answer position for all eight yes/no
surface forms across the three POPE splits (log scale on $y$). The four
right-side forms, including \texttt{yes} (3582) and \texttt{no} (1217),
receive zero first-token mass in any split. The two-token protocol reads
only 3582 and 1217.}
\label{fig:toptokens}
\end{figure}

Figure~\ref{fig:pope_examples} shows eight real adversarial samples, two
per confusion-matrix outcome under the eight-token rule.

\begin{figure}[h]
\centering
\setlength{\fboxsep}{0pt}\setlength{\fboxrule}{1.2pt}
\begin{tabular}{@{}c@{\hspace{4pt}}c@{\hspace{4pt}}c@{\hspace{4pt}}c@{}}
\fcolorbox{eightcolor!70}{white}{\includegraphics[width=0.22\linewidth, height=2.45cm, keepaspectratio]{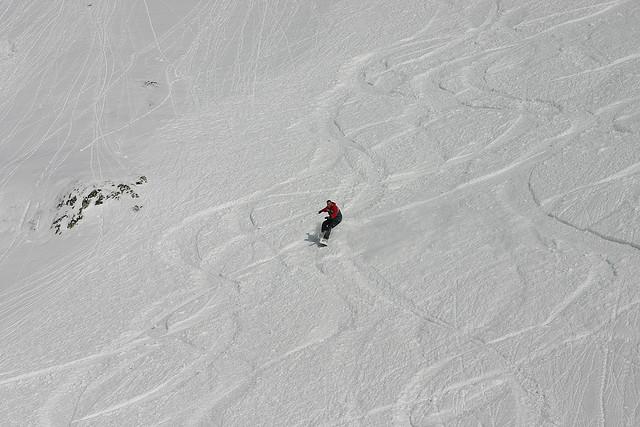}} &
\fcolorbox{eightcolor!70}{white}{\includegraphics[width=0.22\linewidth, height=2.45cm, keepaspectratio]{figures/sample_tp_2.jpg}} &
\fcolorbox{twocolor!70}{white}{\includegraphics[width=0.22\linewidth, height=2.45cm, keepaspectratio]{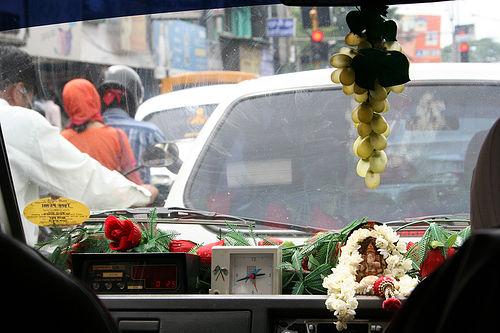}} &
\fcolorbox{twocolor!70}{white}{\includegraphics[width=0.22\linewidth, height=2.45cm, keepaspectratio]{figures/sample_fp_2.jpg}} \\[1pt]
\footnotesize\textcolor{eightcolor}{\textbf{TP}}\,\emph{snowboard?} &
\footnotesize\textcolor{eightcolor}{\textbf{TP}}\,\emph{truck?} &
\footnotesize\textcolor{twocolor}{\textbf{FP}}\,\emph{bus?} &
\footnotesize\textcolor{twocolor}{\textbf{FP}}\,\emph{traffic light?} \\[3pt]
\fcolorbox{eightcolor!70}{white}{\includegraphics[width=0.22\linewidth, height=2.45cm, keepaspectratio]{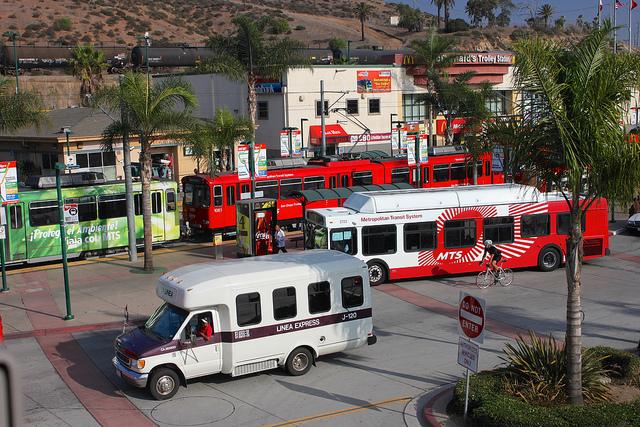}} &
\fcolorbox{eightcolor!70}{white}{\includegraphics[width=0.22\linewidth, height=2.45cm, keepaspectratio]{figures/sample_tn_2.jpg}} &
\fcolorbox{twocolor!70}{white}{\includegraphics[width=0.22\linewidth, height=2.45cm, keepaspectratio]{figures/sample_fn_1.jpg}} &
\fcolorbox{twocolor!70}{white}{\includegraphics[width=0.22\linewidth, height=2.45cm, keepaspectratio]{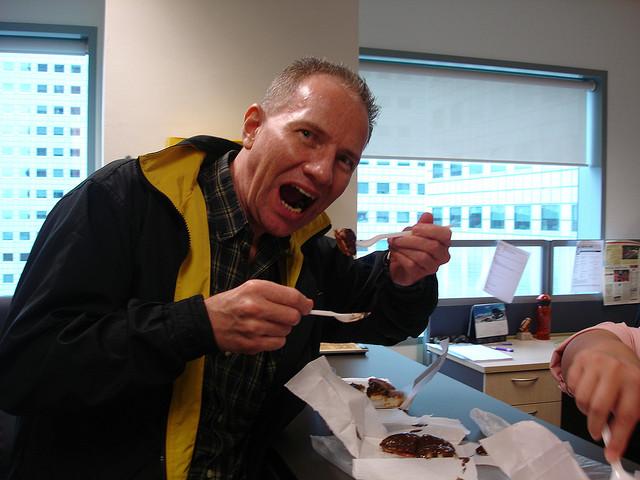}} \\[1pt]
\footnotesize\textcolor{eightcolor}{\textbf{TN}}\,\emph{dining table?} &
\footnotesize\textcolor{eightcolor}{\textbf{TN}}\,\emph{dining table?} &
\footnotesize\textcolor{twocolor}{\textbf{FN}}\,\emph{bed?} &
\footnotesize\textcolor{twocolor}{\textbf{FN}}\,\emph{spoon?} \\
\end{tabular}
\caption{Eight POPE adversarial samples from our run (question IDs 1, 7,
26, 72 on top; 14, 20, 35, 39 on the bottom), two per confusion-matrix
outcome under eight-token evaluation. \textcolor{eightcolor}{Blue}
borders mark predictions that match the label;
\textcolor{twocolor}{red} borders mark errors. The split queries objects
that co-occur with present ones, which is what makes the FP and FN cases
hard to fix with a generic visual-uncertainty signal.}
\label{fig:pope_examples}
\end{figure}

\section{Full Derivation of the Protocol Gap}
\label{app:math}

Table~\ref{tab:notation} fixes the notation.

\begin{table}[h]
\caption{Notation used in the derivation.}
\label{tab:notation}
\centering
\footnotesize
\begin{tabular}{lp{11.4cm}}
\toprule
Symbol & Meaning \\
\midrule
$P,\,N$           & ground-truth positive / negative question counts in the split \\
$T$               & true-yes greedy generations (token 3869 produced when label$=$yes) \\
$G$               & true-negative count: token 1939 produced when label$=$no (the $\mathrm{TN}$ cell under the eight-token rule) \\
$G_{\text{tot}}$  & total token-1939 emissions across the split, $G_{\text{tot}} = G + (P - T)$, because the model emits \texttt{`\,No'} on true-negative and false-negative questions alike \\
$\rho_N$          & fraction of true-negative \texttt{`\,No'} emissions the two-token rule flips to yes, counted per split \\
$\varphi$         & fraction of false-negative \texttt{`\,No'} emissions (label$=$yes) the two-token rule reads as yes, counted per split \\
$\ell_v$          & logit at vocabulary id $v$ (first decoded position) \\
$\hat{y}$         & predicted label after the readout rule is applied \\
\bottomrule
\end{tabular}
\end{table}

The eight-token rule preserves the model's generated decision, so
\begin{equation}
\mathrm{TP}_8 = T,\quad \mathrm{FP}_8 = N - G,\quad
\mathrm{FN}_8 = P - T,\quad \mathrm{TN}_8 = G,
\label{eq:counts-eight}
\end{equation}
where $G$ is the true-negative count, not the total token-1939 count
$G_{\text{tot}}$. With $G = 1{,}318$ and $N = 1{,}500$ on the adversarial
split, $\mathrm{FP}_8 = 182$, matching Figure~\ref{fig:confusion}.

The two-token rule reads two logits the decoder never selects, so it
does not preserve the decision. On a question where the model emits
\texttt{`\,Yes'} it could in principle answer no, but the records show
it never does: across the three splits all 3{,}893 \texttt{`\,Yes'}
emissions are also read as yes (the 8y$\to$2n column of
Table~\ref{tab:transitions} is zero). Given that observed condition the
rule changes only \texttt{`\,No'} emissions. It flips a fraction
$\rho_N$ of the true-negative ones into $\mathrm{FP}$ and a fraction
$\varphi$ of the false-negative ones (label$=$yes questions the model
had answered no) into $\mathrm{TP}$:
\begin{equation}
\mathrm{TP}_2 = T + \varphi(P - T),\quad
\mathrm{FP}_2 = (N - G) + \rho_N\,G,\quad
\mathrm{FN}_2 = (1-\varphi)(P - T),\quad
\mathrm{TN}_2 = (1-\rho_N)\,G.
\label{eq:counts-two}
\end{equation}
Substituting Eqs.~\ref{eq:counts-eight}--\ref{eq:counts-two} into
$\mathrm{F1} = 2\mathrm{TP}/(2\mathrm{TP} + \mathrm{FP} + \mathrm{FN})$
gives the closed form quoted as Eq.~\ref{eq:gap-closed} in the main
text. If an emitted \texttt{`\,Yes'} were ever read as no, a third rate
would enter $\mathrm{TP}_2$ and $\mathrm{FP}_2$; on these records that
rate is zero.

\paragraph{What the records contain.} Each of the 9{,}000 records holds
the question ID, the label, the eight-token and two-token decisions, and
the argmax token ID. The logits behind the decisions were compared in
memory and not saved, so the records fix the sign of
$\ell_{3582} - \ell_{1217}$ on every question but neither value. We
therefore cannot report means, medians or ranks for IDs 3582 and 1217.
A separate pass over the adversarial split, run with eager attention so
that it could return attention maps, saved the maximum logit within the
eight-token yes set and within the no set. Over its 3{,}000 questions
these maxima average 24.63 (yes set) and 24.53 (no set). Each is a
maximum over four IDs, so neither isolates 3582 or 1217. That pass
matches the main run's eight-token decision on 2{,}998 of 3{,}000
questions (it differs on IDs 77 and 2938, both fp16 ties or near-ties;
Appendix~\ref{app:why}), and every count below comes from the main run.

\paragraph{Per-split counts and rates.} Table~\ref{tab:transitions}
pairs the eight-token and two-token decisions question by question. In
the notation of Table~\ref{tab:notation}, $T$ is the 8y$\to$2y count on
the label-yes row, $N - G$ the same count on the label-no row, $G$ the
label-no 8n total and $P - T$ the label-yes 8n total.

\begin{table}[!ht]
\caption{Paired readout transitions for LLaVA-1.5-7B on each POPE split
(3{,}000 questions each, $P{=}N{=}1{,}500$). Each column gives the
eight-token decision followed by the two-token decision. The 8y$\to$2n
column is zero everywhere, so the two-token rule only changes
\texttt{`\,No'} emissions. Both rates are read from the 8n$\to$2y
column.}
\label{tab:transitions}
\centering
\footnotesize
\begin{tabular}{llrrrrcc}
\toprule
Split & Label & 8y$\to$2y & 8y$\to$2n & 8n$\to$2y & 8n$\to$2n & Rate & Value \\
\midrule
\multirow{2}{*}{Adversarial} & yes & 1{,}174 & 0 & 264 & 62      & $\varphi$ & $264/326 = 0.810$ \\
                             & no  & 182     & 0 & 660 & 658     & $\rho_N$  & $660/1{,}318 = 0.501$ \\
\addlinespace
\multirow{2}{*}{Popular}     & yes & 1{,}191 & 0 & 250 & 59      & $\varphi$ & $250/309 = 0.809$ \\
                             & no  & 112     & 0 & 567 & 821     & $\rho_N$  & $567/1{,}388 = 0.409$ \\
\addlinespace
\multirow{2}{*}{Random}      & yes & 1{,}191 & 0 & 250 & 59      & $\varphi$ & $250/309 = 0.809$ \\
                             & no  & 43      & 0 & 448 & 1{,}009 & $\rho_N$  & $448/1{,}457 = 0.307$ \\
\bottomrule
\end{tabular}
\end{table}

On POPE adversarial the direct token-ID counts and the two confusion
matrices fix every term: $T{=}1{,}174$, $G{=}1{,}318$,
$G_{\text{tot}}{=}1{,}644$, $P{=}N{=}1{,}500$, with
$\rho_N = 660/1{,}318 = 0.501$ and $\varphi = 264/326 = 0.810$. The 924
two-token flips split into 660 landing on true-negative questions
(becoming $\mathrm{FP}$) and 264 landing on borderline true-yes questions
(becoming $\mathrm{TP}$):
\begin{equation}
\mathrm{F1}_8 = \frac{2348}{2856} = 0.8221,
\qquad
\mathrm{F1}_2 = \frac{2876}{3780} = 0.7608.
\label{eq:f1-vals}
\end{equation}
The two-rate closed form reproduces all four two-token cells exactly
($\mathrm{TP}_2{=}1{,}438$, $\mathrm{FP}_2{=}842$, $\mathrm{FN}_2{=}62$,
$\mathrm{TN}_2{=}658$), so $\mathrm{F1}_2 = 0.7608$ comes out without any
empirical correction term. Set $\varphi = 0$ and the lower bound is
$\mathrm{F1}_2 = 2348/3516 = 0.6678$; the observed 0.7608 is higher
because $\varphi = 0.810$. The same substitution with each split's own
rates gives
\begin{equation}
\begin{aligned}
\mathrm{F1}_2^{\text{pop}}  &= \frac{2\cdot 1{,}441}{2\cdot 1{,}441 + 679 + 59}
  = \frac{2882}{3620} = 0.7961,\\
\mathrm{F1}_2^{\text{rand}} &= \frac{2\cdot 1{,}441}{2\cdot 1{,}441 + 491 + 59}
  = \frac{2882}{3432} = 0.8397,
\end{aligned}
\label{eq:f1-vals-other}
\end{equation}
against eight-token values of $2382/2803 = 0.8498$ and
$2382/2734 = 0.8713$, matching Table~\ref{tab:protocol}. Applying the
adversarial rates ($\rho_N = 0.501$, $\varphi = 0.810$) to the popular
and random counts instead gives 0.7690 and 0.7761, so the rates have to
be counted per split. All of this is an accounting identity. The rates
are read off the same transitions they are then used to rebuild, so the
reconstruction checks the bookkeeping and carries no independent
evidence about why the two logits fall the way they do.

\paragraph{Pooled flip rate.} The fraction of all \texttt{`\,No'}
emissions the two-token rule reads as yes is $924/1{,}644 = 0.562$ on
adversarial, $817/1{,}697 = 0.481$ on popular and $698/1{,}766 = 0.395$
on random. Each pooled rate is the mix
$\big(\rho_N G + \varphi(P - T)\big)/\big(G + P - T\big)$ of the two
group rates. Since $\varphi$ stays near 0.81, the decline across splits
comes from $\rho_N$. A comparison that ignored the input would flip true
negatives at the same expected rate on every split. Instead $\rho_N$ is
highest on adversarial, whose negatives name objects that co-occur with
present ones, and lowest on random. Together with $\varphi > \rho_N$ on
every split, this pattern is consistent with $\ell_{3582} - \ell_{1217}$
carrying the model's lean toward yes. Without the individual logits we
cannot measure how strong that dependence is.

\section{Formal Equivalence of String-Parse and Single-Token Readouts}
\label{app:equiv}

The observation that \emph{dynamic\_single} and \emph{string\_parse} agree
to four decimals on every run is not a coincidence. Under greedy decoding
the two protocols are identical whenever the argmax token belongs to the
single-token surface-form set, outside the tie case of
Theorem~\ref{thm:equiv}. Writing the proof down turns an empirical
regularity into a claim the rest of the paper can lean on.

\begin{theorem}[Single-token readout equivalence]
\label{thm:equiv}
Let $\mathcal{V}$ be a model's vocabulary and let
$Y \subset \mathcal{V}$, $N \subset \mathcal{V}$ be disjoint single-token
surface-form sets for the yes and no answers. Let
$\boldsymbol{\ell} \in \mathbb{R}^{|\mathcal{V}|}$ be the first-position
logit vector on a fixed input, and $t^\star = \arg\max_v \ell_v$ the
greedy argmax token. Define
\[
\hat{y}_{\mathrm{logit}}
   = \mathbf{1}\!\left[\max_{i \in Y} \ell_i > \max_{j \in N} \ell_j\right],
\qquad
\hat{y}_{\mathrm{str}}
   = \mathbf{1}[\,\textsf{decode}(t^\star) \in \textsf{yes-strings}\,].
\]
If $t^\star \in Y \cup N$ then
$\hat{y}_{\mathrm{logit}} = \hat{y}_{\mathrm{str}}$, unless the maximum
is tied between $Y$ and $N$ with its lowest tied ID in $Y$.
\end{theorem}

\begin{proof}
Greedy decoding emits $t^\star$ as the first token, and \textsf{decode}
returns the surface form whose vocabulary id is $t^\star$. By construction
$Y$ holds exactly the single-token yes forms and $N$ the single-token no
forms. Suppose $t^\star \in Y$. Then $\ell_{t^\star} \geq \ell_v$ for
every $v \in \mathcal{V}$, so $\ell_{t^\star} \geq \max_{j \in N} \ell_j$,
and $\ell_{t^\star} = \max_{i \in Y} \ell_i$ since $t^\star \in Y$. The
strict inequality $\max_Y \ell > \max_N \ell$ holds unless the maximum is
tied between $Y$ and $N$, in which case greedy decoding returns the
lowest tied ID, here $t^\star \in Y$, while the strict rule returns no;
this is the exception in the statement. For the two IDs LLaVA-1.5 emits,
1939 precedes 3869, so a tie between them gives $t^\star = 1939 \in N$,
the case below. Exact ties do occur, because the logits are fp16 values: in
the instrumented adversarial pass (Appendix~\ref{app:why}) the yes-set
and no-set maxima are equal on 8 of the 3{,}000 questions, while the
main sweep's records hold no logits, so ties there cannot be counted.
Outside a tie between $Y$ and $N$, therefore, $\hat{y}_{\mathrm{logit}} = 1$
and $\hat{y}_{\mathrm{str}} = 1$. The case $t^\star \in N$ needs no
exception: then $\max_N \ell = \ell_{t^\star} \geq \max_Y \ell$, so the
strict rule returns no even on a tie, and both readouts give 0.
\end{proof}

\begin{corollary}[Failure mode of the two-token readout]
\label{cor:fail}
Let $Y_2 \subset Y$ and $N_2 \subset N$ be nonempty strict subsets. On a
question whose greedy argmax satisfies
$t^\star \in (Y \cup N) \setminus (Y_2 \cup N_2)$, the readout
$\hat{y}_{\mathrm{logit}}(Y_2, N_2)$ compares two candidates the decoder
did not select. Greedy decoding does not fix which of the two is larger,
so the readout's decision is not tied to the emitted answer: for every
such $t^\star$ there are logit vectors with argmax $t^\star$ on which
$\hat{y}_{\mathrm{logit}}(Y_2, N_2) = \hat{y}_{\mathrm{str}}$, and
others on which the two differ.
\end{corollary}

\begin{proof}
If $t^\star \in Y_2 \cup N_2$, both readouts agree by
Theorem~\ref{thm:equiv} applied to $(Y_2, N_2)$, outside the tie case it
excludes. Now let
$t^\star \in (Y \cup N) \setminus (Y_2 \cup N_2)$. The argmax condition
requires only $\ell_{t^\star} \geq \ell_v$ for every $v \in \mathcal{V}$; it
places no constraint on the order of the logits inside
$Y_2 \cup N_2$. Pick $a \in Y_2$ and $b \in N_2$ (disjoint, because $Y$
and $N$ are) and set $\ell_{t^\star} = 0$, $\ell_a = -1$, $\ell_b = -2$
and $\ell_v = -3$ for every other $v$. Then $t^\star$ is the argmax and
$\hat{y}_{\mathrm{logit}}(Y_2, N_2) = 1$. Swapping the values at $a$ and
$b$ leaves $t^\star$ the argmax and gives
$\hat{y}_{\mathrm{logit}}(Y_2, N_2) = 0$. Since $\hat{y}_{\mathrm{str}}$
depends on $t^\star$ alone, it equals the readout in exactly one of the
two cases.
\end{proof}

The corollary says nothing about how large the excluded logits are or
how often the readout agrees in practice; agreement is an empirical
property of the model and template. The LLaVA-1.5 audit sits entirely in
the second case. The generated token is always 3869 or 1939, and the
two-token readout uses $Y_2 = \{3582\}$, $N_2 = \{1217\}$, so
$(Y \cup N) \setminus (Y_2 \cup N_2)$ covers all 9{,}000 questions and
Theorem~\ref{thm:equiv} guarantees nothing about the two-token decision
on any of them. On the records the readout agrees with the emitted
answer on 2{,}076, 2{,}183 and 2{,}302 of the 3{,}000 questions in the
adversarial, popular and random splits (0.692, 0.728, 0.767). The
agreement is lopsided. It is complete on the 3{,}893 \texttt{`\,Yes'}
emissions and partial on \texttt{`\,No'} emissions, where
$\ell_{3582} > \ell_{1217}$ on 0.81 of the label-yes questions and on
0.307 to 0.501 of the label-no ones (Table~\ref{tab:transitions}). That
asymmetry is what inflates the two-token yes-rate to 0.760 on
adversarial, and it indicates that the excluded logits may track the
model's lean. Across the multi-model runs the
\emph{dynamic\_single} and \emph{string\_parse} readouts report identical
F1, precision, recall, and yes-rate to four decimals on
LLaVA-1.6-Mistral-7B, mPLUG-Owl2, and Qwen2-VL. On InstructBLIP they
disagree, and the disagreement does not contradict
Theorem~\ref{thm:equiv}: the theorem equates the logit readout with
parsing $\textsf{decode}(t^\star)$, a single token, whereas string parsing
here scanned up to six generated tokens of free text whose leading span
did not begin with the answer token.

\section{The Eight-Token Standard}
\label{app:standard}

For LLaMA-family VLMs evaluated on POPE under the
\texttt{USER:\,<image>\textbackslash n\{question\}\textbackslash nASSISTANT:}
template, the readout rule is
\begin{equation}
Y = \{3582, 8241, 4874, 3869\},\quad
N = \{1217, 3782, 694, 1939\},\quad
\hat{y} = \mathbf{1}\!\left[\max_{i \in Y} \ell_i > \max_{j \in N} \ell_j\right].
\label{eq:rule}
\end{equation}
Algorithm~\ref{alg:eight-token} states the same rule in pseudocode and
folds in the diagnostic pass we run for any new model-template pairing.
The diagnostic on lines 2--9 is precisely what we wish a prior evaluation
had run on LLaVA-1.5-7B. Token-set choice is context-dependent: a
different model or template may push first-token mass onto other forms, so
the diagnostic, not the fixed eight-token set, is what ports.

\begin{algorithm}[h]
\caption{Eight-token POPE readout with new-model diagnostic.}
\label{alg:eight-token}
\begin{algorithmic}[1]
\REQUIRE Model $\mathcal{M}$, processor, prompt template $\mathcal{T}$, POPE split $\mathcal{D}$.
\REQUIRE Token sets $Y = \{3582, 8241, 4874, 3869\}$, $N = \{1217, 3782, 694, 1939\}$.
\STATE \textbf{Diagnostic (once per model-template combination):}
\STATE Sample 100 questions $\{q_k\}$ from $\mathcal{D}$.
\FOR{each $q_k$}
  \STATE $\boldsymbol{\ell}_k \gets \mathcal{M}(\mathcal{T}(q_k))$ \quad // first-token logits, \texttt{max\_new\_tokens=1}
  \STATE $t_k \gets \arg\max_v \boldsymbol{\ell}_k[v]$
\ENDFOR
\IF{$\{t_k\}_k \not\subseteq Y \cup N$}
  \STATE Add missing IDs to $Y$ or $N$ before continuing.
\ENDIF
\STATE \textbf{Evaluation:}
\FOR{each question $q \in \mathcal{D}$}
  \STATE $\boldsymbol{\ell} \gets \mathcal{M}(\mathcal{T}(q))$
  \STATE $\hat{y} \gets \mathbf{1}\!\left[\max_{i \in Y} \boldsymbol{\ell}[i] > \max_{j \in N} \boldsymbol{\ell}[j]\right]$
\ENDFOR
\STATE \textbf{return} F1, precision, recall, yes-rate.
\end{algorithmic}
\end{algorithm}

\section{Inference-Time Correction Experiments}
\label{app:correction}

With the baseline pinned at 0.8221, we asked a simple question: can any
inference-time signal push past it on the adversarial split? Over roughly
90 GPU-hours we ran nine variants, six grounding-based and three
contrastive. None improved adversarial F1. Section~\ref{app:why} below
reports the per-question diagnostics that explain why.

\subsection{Methods tested}

All nine variants act on the answer-position logits, before the
eight-token decision rule. Six suppress yes-logits when an external
visual-uncertainty signal is low; one applies uniform suppression as a
control; two use contrastive image variants. Where relevant we sweep the
suppression strength $\beta$.

\begin{itemize}[leftmargin=*,nosep]
\item \textbf{CLIP-B/32 certainty gating.} Cosine similarity between
CLIP-ViT-B/32~\citep{radford2021learning} patch features and the
question-noun embedding gives a certainty in $[0,1]$, used to scale
yes-suppression as
$\ell_{\text{yes}}^{\text{mod}} = \ell_{\text{yes}} - \beta(1 - c)$, with
$\beta \in \{1.0, 1.5\}$.
\item \textbf{CLIP-L/14-336 certainty gating.} Same procedure with the
larger CLIP-L, using patch features through its visual projection head, so
the embedding space matches LLaVA-1.5's own vision tower.
\item \textbf{Entropy gating.} Answer-position token-distribution entropy,
normalised to $[0,1]$; higher entropy means lower certainty and stronger
suppression.
\item \textbf{Magnitude gating.} The relative magnitude of the yes-token
logit in the full vocabulary distribution, used as a soft suppression
threshold.
\item \textbf{CLS feature similarity.} Cosine similarity between the
CLIP-ViT-L/14-336 CLS embedding and the question-noun text embedding.
\item \textbf{Object-noun CLIP similarity.} Full-image CLIP-B/32 embedding
against the question-noun embedding, with rule-based noun extraction from
the question template.
\item \textbf{Uniform yes-suppression.} A control that applies a fixed
suppression to every question, to quantify how much movement is
signal-driven versus a static yes-bias correction.
\item \textbf{Generation midpoint gating.} The yes/no logit gap read at an
intermediate generation step rather than the final position.
\item \textbf{VCD-noise perturbation.} Following
\citet{leng2024mitigating}, a second forward pass with Gaussian noise on
the image embedding gives a contrastive logit set; the decision is taken
on the difference.
\end{itemize}

\subsection{Results}

Table~\ref{tab:correction} reports F1, precision, and recall for all nine.
Figure~\ref{fig:correction_methods} sorts them by F1 and places them on
the precision-recall plane. Every yes-suppression variant buys 1.5 to 6.0
precision points at a cost of 3.0 to 5.8 recall points, and in no case is
the precision gain enough to cover the recall loss.

\begin{table}[h]
\caption{Nine inference-time corrections on POPE adversarial (3{,}000
questions, eight-token rule). The closest, CLIP-B/32 $\beta{=}1.0$, sits
0.005 F1 below baseline.}
\label{tab:correction}
\centering
\begin{tabular}{lccc}
\toprule
Method & F1 & Precision & Recall \\
\midrule
Eight-token baseline & 0.8221 & 0.8658 & 0.7827 \\
\midrule
CLIP-B/32 certainty ($\beta = 1.0$) & 0.8171 & 0.9072 & 0.7433 \\
CLIP-B/32 certainty ($\beta = 1.5$) & 0.8104 & 0.9278 & 0.7193 \\
CLIP-L/14-336 certainty             & 0.8159 & 0.9031 & 0.7453 \\
Entropy gating                      & 0.8164 & 0.9023 & 0.7453 \\
Magnitude gating                    & 0.8148 & 0.8981 & 0.7467 \\
CLS feature similarity              & 0.8153 & 0.8963 & 0.7460 \\
Object-noun CLIP                    & 0.8166 & 0.9017 & 0.7453 \\
Uniform suppression                 & 0.8110 & 0.9221 & 0.7247 \\
VCD-noise perturbation              & 0.8169 & 0.9009 & 0.7440 \\
\bottomrule
\end{tabular}
\end{table}

\begin{figure}[h]
\centering
\begin{subfigure}[t]{0.46\linewidth}
\centering
\includegraphics[width=\linewidth]{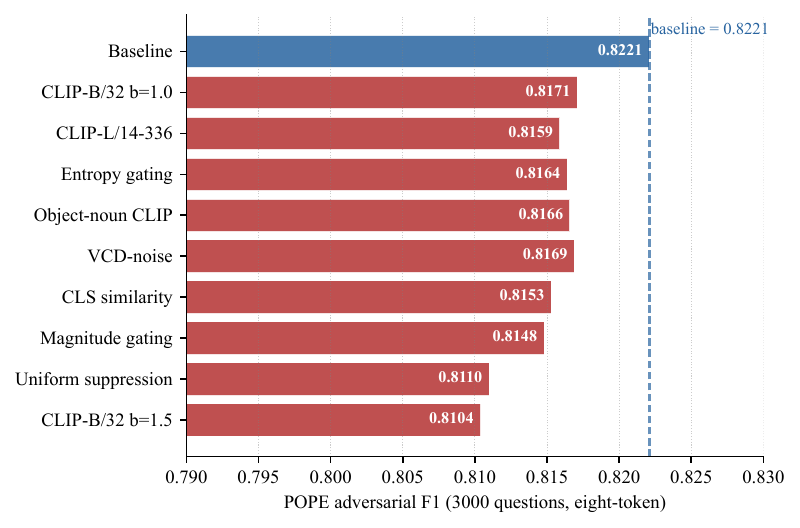}
\caption{F1 ordered against baseline.}
\label{fig:correction_methods_bar}
\end{subfigure}
\hfill
\begin{subfigure}[t]{0.50\linewidth}
\centering
\includegraphics[width=\linewidth]{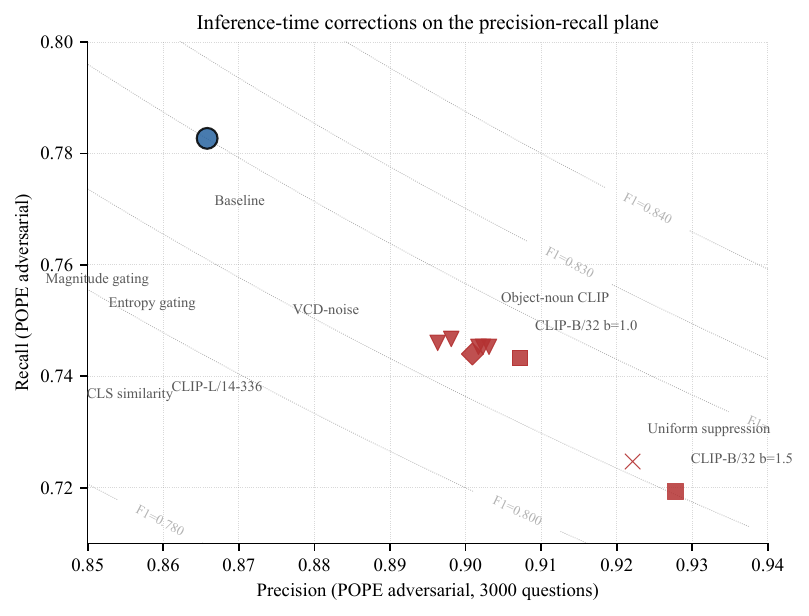}
\caption{Precision--recall plane.}
\label{fig:correction_methods_pr}
\end{subfigure}
\caption{Nine inference-time corrections on POPE adversarial, two views.
\textbf{(a)} F1 against the corrected baseline at 0.8221. \textbf{(b)} The
same methods on the precision-recall plane with F1 iso-curves. Every
yes-suppression variant moves north and west of the baseline; no path
crosses the $\mathrm{F1}{=}0.8221$ iso-curve.}
\label{fig:correction_methods}
\end{figure}

\subsection{Why the adversarial split resists correction}
\label{app:why}

To see the failure mode, we ran the full 3{,}000-question adversarial
diagnostic, instrumenting every signal that featured in the nine methods
and binning each question by its TP/FP/TN/FN outcome. The baseline splits
into 1{,}173 TPs, 1{,}317 TNs, 183 FPs, and 327 FNs (2{,}490 correct,
510 wrong). This diagnostic was a separate instrumented run, and it
differs from the main sweep (1{,}174 / 1{,}318 / 182 / 326) on 2 of the
3{,}000 questions, both fp16 ties or near-ties: on ID 77 (label yes) the
yes-set and no-set maxima are both 24.6875 and the diagnostic's
strict-inequality rule answers no, and on ID 2938 (label no) the yes-set
maximum leads by 0.0156, one fp16 step at this magnitude. Three signals
carry most of the weight.

\paragraph{CLIP-L similarity separates TPs from FPs by a small but
reliable margin.} We took the maximum cosine similarity between the
question-noun text embedding and CLIP-L's 576 patch features for every
question. The TP group ($n{=}1{,}173$) has mean 0.1507 and the FP group
($n{=}183$) has mean 0.1375, a gap of 0.0132 (95\% bootstrap CI
$[0.0084, 0.0180]$, 10{,}000 resamples, percentile method, NumPy
\texttt{default\_rng} seed 42; Mann--Whitney $U=132923$,
$p=2.05\times10^{-7}$; Cohen's $d=0.40$). Reliable but small. A gate keyed
on this signal would have to separate TPs from FPs at $d=0.40$ while the
precision-recall tradeoff punishes any threshold that suppresses FPs at
the expense of TPs. The pattern fits a structural property of the split:
adversarial negatives co-occur in the same image as the queried object,
CLIP was trained on the same co-occurrence statistics, and its embedding
space encodes them.

\paragraph{Decoder attention points the wrong way.} Mean attention from
the final text token to the 576 image positions, averaged across layers
14--20, is 0.0738 for correct predictions ($n{=}2{,}490$) and 0.0801 for
wrong ones ($n{=}510$); Mann--Whitney $U=452635$, two-sided
$p=1.45\times10^{-24}$, Cohen's $d=-0.46$. The layers-14--20 window was
fixed before analysis as the mid-to-late band of the 32-layer LLaMA-2
decoder, not tuned to maximise the effect. The direction holds per
category: FPs show higher mean attention (0.0833) than TPs (0.0780;
$p=4.13\times10^{-9}$), and FNs higher (0.0783) than TNs (0.0700;
$p=1.50\times10^{-23}$). Wrong predictions consistently draw more image
attention than correct ones (Figure~\ref{fig:grounding}). A plausible
reading is that uncertain predictions trigger more looking as the model
hunts for evidence, which runs against the ``more grounding means more
correct'' assumption the six grounding gates rely on.

\begin{figure}[h]
\centering
\includegraphics[width=0.72\linewidth]{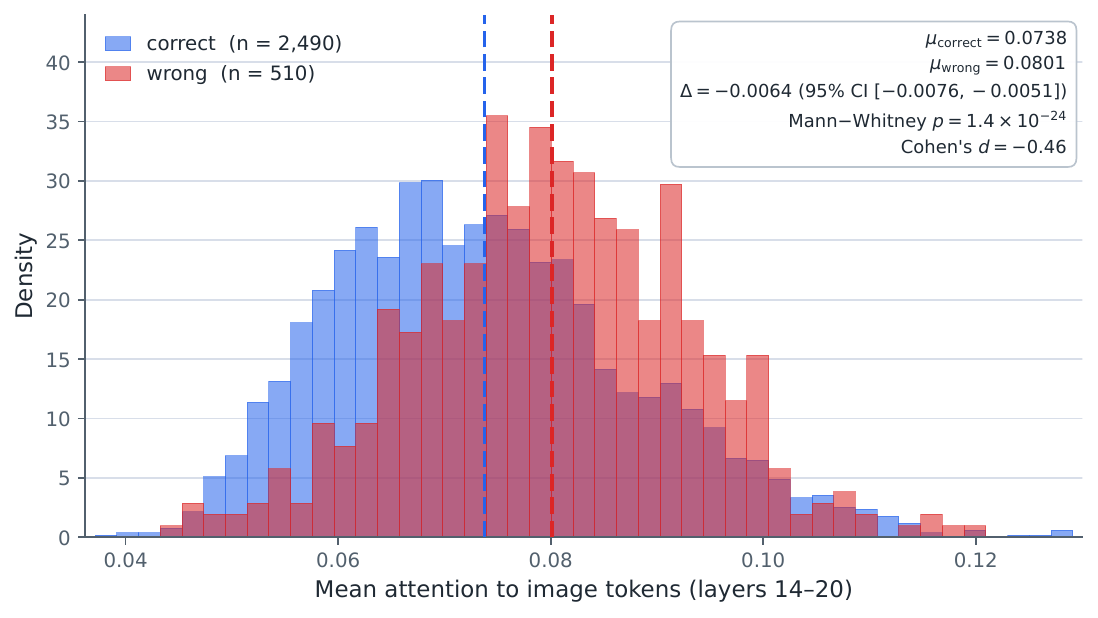}
\caption{Distribution of mean attention from the final text token to image
positions, layers 14--20, on the 3{,}000-question adversarial run. Blue
bars are correct ($n{=}2490$); red bars are wrong ($n{=}510$). Higher
image attention associates with wrong predictions, so suppressing
yes-logits when attention is low removes signal from the wrong side of the
distribution.}
\label{fig:grounding}
\end{figure}

\paragraph{The logit gap separates by confidence, not correctness.}
Figure~\ref{fig:signals} gives the mean absolute logit gap
$|\ell_{\text{yes}} - \ell_{\text{no}}|$ per category: 2.066 for TPs,
1.482 for TNs, 0.700 for FPs, and 0.962 for FNs. It separates correct from
wrong with Cohen's $d=1.06$ across all 3{,}000 questions
($p=1.22\times10^{-93}$). Within the positive class TPs beat FNs
($d=1.22$, $p=4.83\times10^{-68}$); within the negative class TNs beat FPs
($d=1.18$, $p=1.19\times10^{-41}$). The model is most confident when
right and least confident when wrong. Suppression acts on small-gap
questions in both directions at once: a small-gap TP gets pushed to FN at
roughly the rate a small-gap FP gets pushed to TN, so the net F1
contribution is near zero, and in our runs the sign came out negative
every time. The $p$-values here are uncorrected for multiple comparisons;
given the effect magnitudes, a Bonferroni correction over the six tests
changes no qualitative conclusion.

\begin{figure}[h]
\centering
\includegraphics[width=0.92\linewidth]{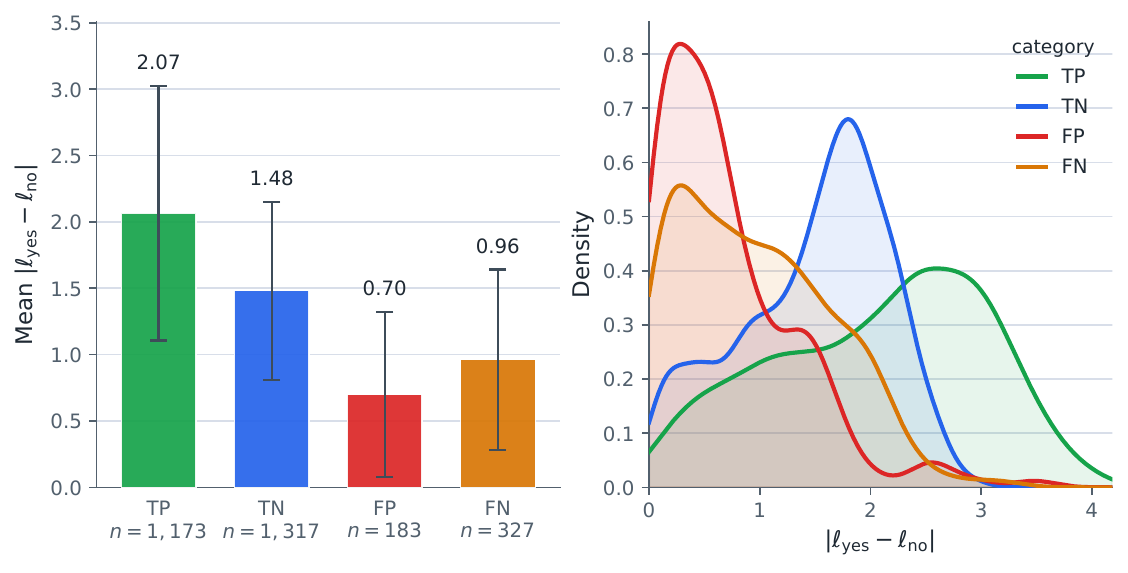}
\caption{Absolute logit gap $|\ell_{\text{yes}} - \ell_{\text{no}}|$ at
the answer position on the 3{,}000-question adversarial run. Left:
per-category mean with one-standard-deviation error bars (TP $n{=}1173$,
TN $n{=}1317$, FP $n{=}183$, FN $n{=}327$; means 2.07, 1.48, 0.70, 0.96).
Right: kernel-density estimates. The order TP $>$ TN $>$ FN $>$ FP at the
mean separates confidence from correctness, which is why uncertainty-keyed
corrections trade TPs for FNs at about the same rate as FPs for TNs.}
\label{fig:signals}
\end{figure}

The adversarial split is behaving as designed. Co-occurring negatives are
hard to refute with image evidence, because the evidence ambiguously
supports both the queried object and the one actually present. The methods
we tested assume a separable correctness signal the split is built to
deny. On the easier splits we would expect the same suppression to look
better, since those negatives are easier to rule out visually. Beating the
adversarial baseline at the decoding step would need a signal the dataset
does not undercut, and we do not have a candidate.

Group-level means and standard deviations behind these effect sizes are in
Table~\ref{tab:groupsd}. The pooled attention SD is 0.0140, giving
$d = (0.0738 - 0.0801)/0.0140 = -0.45$ (reported as $-0.46$ at the
precision used above); the correct-versus-wrong logit-gap pooled SD is
0.839, giving $d=1.06$.

\begin{table}[h]
\caption{Per-group means and standard deviations underlying the
effect sizes. SD is the sample standard deviation.}
\label{tab:groupsd}
\centering
\begin{tabular}{llrrr}
\toprule
Signal & Group & $n$ & Mean & SD \\
\midrule
Attention & correct & 2{,}490 & 0.0738 & 0.0142 \\
Attention & wrong   &   510 & 0.0801 & 0.0128 \\
\midrule
Logit gap & TP & 1{,}173 & 2.066 & 0.959 \\
Logit gap & TN & 1{,}317 & 1.482 & 0.670 \\
Logit gap & FP &   183 & 0.700 & 0.621 \\
Logit gap & FN &   327 & 0.962 & 0.681 \\
\bottomrule
\end{tabular}
\end{table}

\section{Multi-Model Detail}
\label{app:crossmodel}

\subsection{Token-set divergence: the LLaVA-1.6-Mistral map}

To expose the mechanism behind the \emph{legacy\_8tok} $=0$ result, we
zoom in on LLaVA-1.6-Mistral, whose Mistral vocabulary makes the LLaMA-2
IDs point at the wrong surface forms.
Figure~\ref{fig:crossmodel-tokens} shows the difference at a glance and
Table~\ref{tab:crossmodel-tokens} lists the tokenizer-derived IDs.

\begin{figure}[h]
\centering
\includegraphics[width=0.92\linewidth]{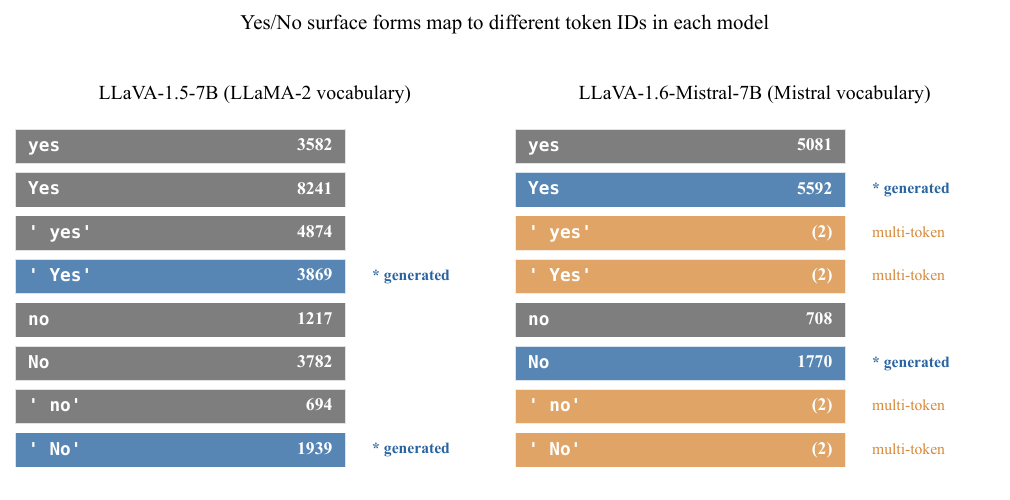}
\caption{Yes/no surface forms map to different token IDs under each model.
\textcolor{eightcolor}{Blue} rows are the forms the model actually
generates first after \texttt{ASSISTANT:}; \textcolor{neutralgray}{grey}
rows are valid single-token IDs that never appear; orange
\emph{multi-token} rows are space-prefixed forms the Mistral tokenizer
splits into two tokens, so they cannot enter a first-token comparison.}
\label{fig:crossmodel-tokens}
\end{figure}

\begin{table}[h]
\caption{Tokenizer-derived single-token IDs for yes/no surface forms under
\texttt{llava-hf/llava-v1.6-mistral-7b-hf}. The four space-prefixed forms
tokenize to two tokens each (\texttt{[28705, ID]}) and are excluded from
the dynamic single-token readout.}
\label{tab:crossmodel-tokens}
\centering
\begin{tabular}{lll}
\toprule
Form & Tokenization & In dynamic set? \\
\midrule
\texttt{yes}  & \texttt{[5081]} & yes (single token) \\
\texttt{Yes}  & \texttt{[5592]} & yes (single token) \\
\texttt{`\,yes'} & \texttt{[28705, 5081]} & no (multi-token) \\
\texttt{`\,Yes'} & \texttt{[28705, 5592]} & no (multi-token) \\
\midrule
\texttt{no}   & \texttt{[708]}  & yes (single token) \\
\texttt{No}   & \texttt{[1770]} & yes (single token) \\
\texttt{`\,no'}  & \texttt{[28705, 708]}  & no (multi-token) \\
\texttt{`\,No'}  & \texttt{[28705, 1770]} & no (multi-token) \\
\bottomrule
\end{tabular}
\end{table}

\subsection{Prompt-template robustness}

The master table fixes one template per model. To check the dynamic
readout is stable under a template change, rather than an artifact of the
LLaVA-1.5 string, we compared the first 500 POPE adversarial questions of
LLaVA-1.6-Mistral under two templates: the LLaVA-1.5 paper template and
the model's native chat template via
\texttt{processor.apply\_chat\_template}. In both runs the only first
tokens generated were id 5592 (\texttt{Yes}) and id 1770 (\texttt{No}); the
LLaMA-2 IDs never appeared. Table~\ref{tab:crossmodel-eval} reports the
four readouts.

\begin{table}[h]
\caption{Four readouts on the first 500 POPE adversarial questions under
LLaVA-1.6-Mistral-7B. \emph{dynamic\_single} compares max-logits over the
tokenizer-derived IDs $\{5081, 5592\}$ vs $\{708, 1770\}$;
\emph{string\_parse} decodes up to 6 tokens and parses the first
\texttt{yes}/\texttt{no} substring.}
\label{tab:crossmodel-eval}
\centering
\small
\begin{tabular}{llcccc}
\toprule
Prompt template & Readout & F1 & Precision & Recall & Yes-rate \\
\midrule
\multirow{4}{*}{Paper template}
 & legacy\_2tok       & 0.6667 & 0.5000 & 1.0000 & 1.000 \\
 & legacy\_8tok       & 0.0000 & 0.0000 & 0.0000 & 0.000 \\
 & dynamic\_single    & 0.8719 & 0.8352 & 0.9120 & 0.546 \\
 & string\_parse      & 0.8719 & 0.8352 & 0.9120 & 0.546 \\
\midrule
\multirow{4}{*}{Native chat template}
 & legacy\_2tok       & 0.6925 & 0.5297 & 1.0000 & 0.944 \\
 & legacy\_8tok       & 0.0000 & 0.0000 & 0.0000 & 0.000 \\
 & dynamic\_single    & 0.8621 & 0.9517 & 0.7880 & 0.414 \\
 & string\_parse      & 0.8621 & 0.9517 & 0.7880 & 0.414 \\
\bottomrule
\end{tabular}
\end{table}

The dynamic single-token and string-parse readouts agree to four decimals
under either template, the empirical face of Theorem~\ref{thm:equiv}. The
legacy eight-token readout is 0 under both, because none of the LLaMA-2
IDs is ever the argmax on a Mistral vocabulary, and the comparison between
two zero-mass logits resolves the same way on every question. Absolute F1
is not comparable to LLaVA-1.5: on the full 3{,}000-question split the
dynamic readout gives 0.8521, and the 0.8719 here is the 500-question
slice. The point of the cross-model run is narrow, that the framework
ports cleanly when token sets are re-derived from each model's tokenizer,
and that transferring the LLaVA-1.5 IDs verbatim is unsafe.

\paragraph{Provenance and verification.} Per-question logs for
LLaVA-1.6-Mistral-7B are in \texttt{experiments/cross\_model/}. For
mPLUG-Owl2, InstructBLIP-Vicuna-7B, and Qwen2-VL-7B-Instruct, the
Table~\ref{tab:multimodel} metrics are the original-run summaries in
\texttt{experiments/multi\_model/}, with a Kaggle notebook that
regenerates them. A secondary verification run (transformers
v4.57.1/v5.10.1) reproduced Qwen2-VL exactly on all splits and mPLUG-Owl2
on adversarial and popular; InstructBLIP (adversarial) differed by about
0.004 F1 under the newer environment, and one mPLUG-Owl2 (random)
regeneration hit a storage error and produced a degenerate all-yes output.
We report the original metrics for a single coherent environment; none of
the qualitative conclusions change.

\section{Literature Survey and Comparability}
\label{app:survey}

This section scopes the finding honestly: the published numbers below
cannot be placed on our scale, because no surveyed paper states its
readout and our reference ran under 4-bit NF4 quantization.
Table~\ref{tab:survey} lists the papers we have inspected that report a
LLaVA-1.5-7B POPE adversarial baseline. For each we record the reported
baseline F1, the reported result, the reported gain, and what the paper
says about the readout (\emph{not stated} if it does not specify, never
inferred from the baseline value).

\begin{table}[h]
\caption{Papers reporting LLaVA-1.5-7B POPE adversarial baselines,
verified entries as of submission. Values are as reported by each paper
(F1$\times100$).}
\label{tab:survey}
\centering
\small
\setlength{\tabcolsep}{4pt}
\begin{tabular}{@{}lcccl@{}}
\toprule
Paper & Rep.~base. & Rep.~res. & Rep.~gain & Readout \\
\midrule
Tang~et~al.~\citep{visflow2025} (adv.)  & 73.54 & 84.35 & $+10.81$ & not stated \\
Tang~et~al.~\citep{visflow2025} (pop.)  & 76.53 & 87.09 & $+10.56$ & not stated \\
Tang~et~al.~\citep{visflow2025} (rand.) & 81.54 & 89.55 & $+8.01$  & not stated \\
Seo~et~al.~\citep{seo2025} (adv.)       & 79.30 & 80.00 & $+0.70$  & not stated \\
Chen~et~al.~\citep{adaptvis2025} (adv.) & 81.00 & 81.80 & $+0.80$  & not stated \\
Huo~et~al.~\citep{huo2024sid} (adv.)    & 80.90 & 82.50 & $+1.60$  & not stated \\
Jiang~et~al.~\citep{jiang2026pnd} (adv.)& 75.06 & 80.82 & $+5.76$  & not stated \\
Jana~et~al.~\citep{jana2026bracs} (adv.)& 82.07 & 83.03 & $+0.96$  & not stated \\
\bottomrule
\end{tabular}
\end{table}

We deliberately decline to headline a claim of the form ``$N$ papers'
gains evaporate.'' Every surveyed paper leaves its readout \emph{not
stated}, so none can be confirmed from released code or explicit
statement, and the number of papers whose readout we can verify is zero.
We therefore report no ratio over verified readouts. Inferring a readout
from a baseline value would break our own rule, and substituting our 82.21
for a paper's baseline would not estimate a corrected gain: our reference
ran under 4-bit NF4 quantization, and a published baseline may differ from
it in readout, precision, prompt format, or decoding. The adversarial
baselines in Table~\ref{tab:survey} run from 73.54 to 82.07, a spread of
8.53 points that these papers do not let us decompose. Settling the
comparison takes matched reruns of each method with the readout logged and
the precision fixed, or at least the readout stated in the paper or its
code. Expanding the survey means reading each paper's evaluation code
rather than inferring the readout.

\section{Diagnostic Script and Per-Split Confusion Matrices}
\label{app:diagnostic}

\begin{verbatim}
python src/run_pope_2tok_baseline.py --split adversarial
\end{verbatim}
The output JSON carries a \texttt{top\_tokens\_generated} field that sorts
first-token IDs by frequency. Any ID appearing there but absent from the
active readout token set must be added before the rule can be trusted on
that model-template pairing.

Table~\ref{tab:confusion} reports full $2\times2$ matrices under the
eight-token rule. The adversarial false-positive rate ($182/1500 = 12.1\%$)
is roughly $4.2\times$ that of the random split ($43/1500 = 2.9\%$),
consistent with the co-occurrence design of the adversarial negatives.

\begin{table}[h]
\caption{Confusion matrices under eight-token evaluation, full
3{,}000-question splits. Counts are exact, taken from the per-question
prediction JSON files; row sums equal 3{,}000 per split.}
\label{tab:confusion}
\centering
\begin{tabular}{lrrrr}
\toprule
Split & TP & TN & FP & FN \\
\midrule
Adversarial & 1{,}174 & 1{,}318 & 182 & 326 \\
Popular     & 1{,}191 & 1{,}388 & 112 & 309 \\
Random      & 1{,}191 & 1{,}457 &  43 & 309 \\
\bottomrule
\end{tabular}
\end{table}

\end{document}